\documentclass{article} 
\usepackage{times}
\usepackage{amsfonts}
\usepackage{amsmath}
\usepackage{graphicx} 
\usepackage{subfigure} 

\newcommand*{\defeq}{\stackrel{\text{def}}{=}}

\usepackage{natbib}

\usepackage{amsthm}

\newtheorem{corollary}{Corollary}
\newtheorem{theorem}{Theorem}

\newif\ifenoughspace
\enoughspacefalse

\title{Estimating or Propagating Gradients Through Stochastic Neurons}

\author{
Yoshua Bengio\\
Department of Computer Science and Operations Research\\
University of Montreal\\
Montreal, H3C 3J7 \\
}

\newcommand{\sigm}{\sigma}

\newcommand{\one}{\mathbf{1}}

\begin{document}

\maketitle

\begin{abstract}
  Stochastic neurons can be useful for a number of reasons in deep learning models,
but in many cases they pose a challenging problem: how to estimate the gradient of a
loss function with respect to the input of such stochastic neurons, i.e., can
we ``back-propagate'' through these stochastic neurons?
We examine this question, existing approaches,
and present two novel families of solutions, applicable in different settings.
In particular, it is demonstrated that a simple biologically plausible formula gives
rise to an an unbiased (but noisy) estimator of the gradient with respect
to a binary stochastic neuron firing probability. Unlike other estimators which
view the noise as a small perturbation in order to estimate gradients by finite
differences, this estimator is unbiased even without assuming that the stochastic
perturbation is small. This estimator is also interesting
because it can be applied in very general settings which do not allow gradient
back-propagation, including the estimation of the gradient with respect to future rewards,
as required in reinforcement learning setups. We also propose an approach to approximating
this unbiased but high-variance estimator by learning to predict it using a biased
estimator. The second approach we propose assumes
that an estimator of the gradient can be back-propagated and it provides an unbiased
estimator of the gradient, but can only work with non-linearities unlike the hard threshold,
but like the rectifier, that are not flat for all of their range.
This is similar to traditional sigmoidal units but has the advantage that
for many inputs, a hard decision (e.g., a 0 output) can be produced, which would
be convenient for conditional computation and achieving sparse representations
and sparse gradients.
\end{abstract}

\section{Introduction and Background}

Many learning algorithms and in particular those based on neural networks 
or deep learning rely on gradient-based learning. To compute exact gradients,
it is better if the relationship between parameters and the training objective
is continuous and generally smooth. If it is only constant by parts, i.e., mostly flat, then
gradient-based learning is impractical. This was what motivated the
move from neural networks made of so-called formal neurons, with a hard
threshold output, to neural networks whose units are based on a
sigmoidal non-linearity, and the well-known back-propagation 
algorithm to compute the gradients~\citep{Rumelhart86b}.

We call {\em computational graph} or {\em flow graph} the graph that
relates inputs and parameters to outputs and training criterion.  Although
it had been taken for granted by most researchers that smoothness of this
graph was a necessary condition for exact gradient-based training methods
to work well, recent successes of deep networks with rectifiers and other
``semi-hard''
non-linearities~\citep{Glorot+al-AI-2011-small,Krizhevsky-2012-small,Goodfellow+al-ICML2013-small}
clearly question that belief: see Section~\ref{sec:semi-hard} for a deeper
discussion.

In principle, even if there are hard decisions (such as the
treshold function typically found in formal neurons) in the computational
graph, it is possible to obtain {\em estimated gradients} by introducing
{\em perturbations} in the system and observing the effects. Although
finite-difference approximations of the gradient 
appear hopelessly inefficient (because independently perturbing
each of $N$ parameters to estimate its gradient would be $N$ times
more expensive than ordinary back-propagation), another option
is to introduce {\em random perturbations}, and this idea has
been pushed far (and experimented on neural networks for control)
by \citet{Spall-1992} with the
Simultaneous Perturbation Stochastic Approximation (SPSA) algorithm.

As discussed here (Section~\ref{sec:semi-hard}), semi-hard non-linearities
and stochastic perturbations can be combined to obtain reasonably
low-variance estimators of the gradient, and a good example of that success
is with the recent advances with {\em
  dropout}~\citep{Hinton-et-al-arxiv2012,Krizhevsky-2012,Goodfellow+al-ICML2013-small}. The
idea is to multiply the output of a non-linear unit by independent binomial
noise.  This noise injection is useful as a regularizer and it does slow
down training a bit, but not apparently by a lot (maybe 2-fold), which is very
encouraging. The symmetry-breaking and induced sparsity may also compensate
for the extra variance and possibly help to reduce ill-conditioning, as
hypothesized by~\citet{Bengio-arxiv-2013}.

However, it is appealing to consider noise whose amplitude can be modulated
by the signals computed in the computational graph, such as with
{\em stochastic binary neurons}, which output a 1 or a 0 according
to a sigmoid probability. Short of computing an average over an
exponential number of configurations, it would seem that computing the
exact gradient (with respect to the average of the loss over all possible
binary samplings of all the stochastic neurons in the neural network)
is impossible in such neural networks. The question is whether good
estimators (which might have bias and variance) can be computed. We
show in Section~\ref{sec:unbiased} that one can indeed produce
an unbiased estimator of the gradient, and that this gradient is
even cheaper to compute than with the back-propagation algorithm, 
since it does not require a backwards pass. To compensate its possibly
high variance, we also propose in Section~\ref{sec:biased} a biased
but lower-variance estimator that is based on {\em learning a correction
factor to deterministically transform a biased but low-variance estimator 
into a less biased one, with the same variance}.

In a first attempt to study the question of the efficiency of such
estimators, in Section~\ref{sec:correlation}, we first show that the proposed
unbiased estimator can be seen as basically estimating the correlation
between the perturbation and the training loss. We then show that the
Boltzmann machine gradient can also be interpreted as a form of correlation-based
estimator, but missing an additive normalization. 
This is encouraging, since various Boltzmann machines (in
particular the restricted Boltzmann machine) have been quite successful in
recent years~\citep{Hinton06-small,Bengio-2009-book}.

\subsection{More Motivations}

One motivation for studying stochastic neurons is that stochastic behavior
may be a required ingredient in {\em modeling biological neurons}. The
apparent noise in neuronal spike trains could come from an actual noise
source or simply from the hard to reproduce changes in the set of input
spikes entering a neuron's dendrites. Until this question is resolved by
biological observations, it is interesting to study how such noise, which
has motivated the Boltzmann machine~\citep{Hinton84}, may impact
computation and learning in neural networks.

Stochastic neurons with binary outputs are also interesting because
they can easily give rise to {\em sparse representations} (that have many zeros), 
a form of regularization that has been used in many representation
learning algorithms~\citep{Bengio-2009-book,Bengio-Courville-Vincent-TPAMI2013}.
Sparsity of the representation corresponds to the prior that, for a given
input scene $x$, most of the explanatory factors are irrelevant (and that
would be represented by many zeros in the representation). 

Semi-hard stochastic neurons such as those studied in Section~\ref{sec:semi-hard}
also give rise to {\em sparse gradients}, i.e., such that for most
examples, the gradient vector (with respect to parameters) has many zeros.
Indeed, for weights into units that are shut off or are in a flat
saturation region (e.g., at 0 or 1), the derivative will be zero.
As argued in \citet{Bengio-arxiv-2013}, sparse gradients may be useful
to reduce the optimization difficulty due to ill-conditioning, by
reducing the number of interactions between parameters to those parameters
that are simultaneously ``active'' (with a non-zero gradient) for a particular
example.

As argued by~\citet{Bengio-arxiv-2013}, sparse representations may
be a useful ingredient of {\em conditional computation}, by which
only a small subset of the model parameters are ``activated'' (and need
to be visited) for any particular example, thereby greatly reducing the
number of computations needed per example. Sparse gating units may
be trained to select which part of the model actually need to be computed
for a given example.

Binary representations are also useful as keys for a hash table, as in the
semantic hashing algorithm~\citep{Salakhutdinov+Geoff-2009}.  In the
latter, zero-mean Gaussian noise is added to the pre-sigmoid activation of
the code neuron (whose output will be used as hash keys). As the amount of
noise is increased, this forces the weights and biases to also increase to
avoid losing too much information in the sigmoid output, thereby gradually
forcing these units into saturation, i.e., producing a nearly 0 or nearly 1
output. Hence one option would be to train hard-decision neurons by
gradually annealing the magnitude of the parameters, turning soft-output
easy to train units into hard-decision binary-output units. However, it
means that during training, we would not be able to take advantage of the hard zeros 
(only available at the end of training). Since both conditional computatation
and sparse gradients are motivated by speeding up training, this would not
be an interesting solution.

Trainable stochastic neurons would also be useful inside recurrent networks
to take hard stochastic decisions about temporal events at different
time scales. This would be useful to train multi-scale temporal hierarchies
\footnote{Daan Wierstra, personal communication} such that back-propagated
gradients could quickly flow through the slower time scales. Multi-scale
temporal hierarchies for recurrent nets have already been proposed and
involved exponential moving averages of different time
constants~\citep{ElHihi+Bengio-nips8}, where each unit is still updated
after each time step.  Instead, identifying key events at a high level of
abstraction would allow these high-level units to only be updated when
needed (asynchronously), creating a fast track (short path) for gradient
propagation through time.

\subsection{Prior Work}

The idea of having stochastic neuron models is of course very old, with one of the major
family of algorithms relying on such neurons being the Boltzmann
machine~\citep{Hinton84}. In Section~\ref{sec:BM} we study a connection
between Boltzmann machine log-likelihood gradients and perturbation-based
estimators of the gradient discussed here.

Another biologically motivated proposal for synaptic strength learning
was proposed by \citet{Fiete+Seung-2006}. It is based on small zero-mean
{\em i.i.d.} perturbations applied at each stochastic neuron potential
(prior to a non-linearity) and a Taylor expansion of the expected
reward as a function of these variations. \citet{Fiete+Seung-2006}
end up proposing a gradient estimator that looks like a {\em correlation
between the reward and the perturbation}, just like what we obtain
in Section~\ref{sec:unbiased}. However, their estimator is only
unbiased in the limit of small perturbations.

Gradient estimators based on stochastic perturbations have been shown
for long~\citep{Spall-1992} to be much more efficient than standard finite-difference
approximations. Consider $N$ quantities $a_i$ to 
be adjusted in order to minimize an expected loss $L(a)$. A finite difference
approximation is based on measuring separately the effect of changing each
one of the parameters, e.g., through $\frac{L(a) - L(a - \epsilon e_i)}{\epsilon}$,
or even better, through $\frac{L(a + \epsilon e_i) - L(a - \epsilon e_i)}{2 \epsilon}$,
where $e_i = (0,0,\cdots,1,0,0,\cdots,0)$ where the 1 is at position $i$. With $N$
quantities (and typically $O(N)$ computations to calculate $L(a)$), the
computational cost of the gradient estimator is $O(N^2)$. Instead, a perturbation-based
estimator such as found in 
Simultaneous Perturbation Stochastic Approximation (SPSA) \citep{Spall-1992}
chooses a random perturbation vector $z$ (e.g., isotropic Gaussian noise
of variance $\sigma^2$) and estimates the gradient of the expected loss with
respect to $a_i$ through $\frac{L(\theta + z) - L(\theta - z)}{2 z_i}$.
So long as the perturbation does not put too much probability around 0,
this estimator is as efficient as the finite-difference estimator but
requires $O(N)$ less computation. However, like the algorithm proposed
by \citet{Fiete+Seung-2006} this estimator becomes unbiased only
as the perturbations go towards 0. When we want to consider all-or-none
perturbations (like a neuron sending a spike or not), it is not clear
if these assumptions are appropriate. The advantage of the unbiased
estimator proposed here is that it does not require that the perturbations
be small.

Another estimator of the expected gradient through stochastic neurons
was proposed by \citet{Hinton-Coursera2012}
in his lecture 15b. The idea is simply to back-propagate
through the hard threshold function (1 if the argument is positive, 0 otherwise)
as if it had been the identity function. It is clearly a biased estimator,
but when considering a single layer of neurons, it has the right sign
(this is not guaranteed anymore when back-propagating through more hidden layers).

\section{Semi-Hard Stochastic Neurons}
\label{sec:semi-hard}

One way to achieve gradient-based learning in networks of stochastic neurons 
is to build an architecture in which noise is injected so that gradients
can sometimes flow into the neuron and can then adjust it (and its predecessors
in the computational graph) appropriately.

In general, we can consider the output $h_i$ of a stochastic neuron as the
application of a determinstic function that depends on its inputs $x_i$
(typically, a vector containing the outputs of other neurons), its internal parameters
$\theta_i$ (typically the bias and incoming weights of the neuron) 
and on a noise source $z_i$:
\begin{equation}
\label{eq:noisy-output}
  h_i = f(x_i,z_i,\theta_i).
\end{equation}
So long as $f(x_i,z_i,\theta_i)$ in the above equation has a non-zero
gradient with respect to $x_i$ and $\theta_i$, gradient-based learning 
(with back-propagation to compute gradients) can proceed.

For example, if the noise $z_i$ is added or multiplied somewhere in the
computation of $h_i$, gradients can be computed as usual. Dropout
noise~\citep{Hinton-et-al-arxiv2012} and masking noise (in denoising
auto-encoders~\citep{VincentPLarochelleH2008-small}) is multiplied (just
after the neuron non-linearity), while in semantic
hashing~\citep{Salakhutdinov+Geoff-2009} noise is added (just before the
non-linearity). For example, that noise can be binomial (dropout, masking
noise) or Gaussian (semantic hashing).

However, if we want $h_i$ to be binary (like in stochastic binary neurons), 
then $f$ will have derivatives that are 0 almost everywhere (and infinite
at the threshold), so that gradients can never flow. 

There is an intermediat option, that we put forward here: choose $f$ so
that it has two main kinds of behavior, with zero derivatives in some
regions, and with significantly non-zero derivatives in other regions.  
We call these two states of the neuron respectively the insensitive state
and the sensitive state. 

A special case is when the insensitive state corresponds to $h_i=0$ and we
have sparsely activated neurons. The prototypical example of that situation
is the rectifier unit~\citep{Hinton2010,Glorot+al-AI-2011-small}, whose
non-linearity is simply $\max(0,{\rm arg})$. For example,
\[
  h_i = \max(0, z_i + b_i + \sum_j W_{ij} x_{ij})
\]
where $z_i \sim {\cal N}(0,\sigma^2)$ is 0-mean Gaussian noise,
$\theta_i=(b_i,W_{i1},W_{i2},\ldots)$ and $x_{ij}$ is the $j$-th input of
unit $i$, either a raw input (visible unit) or the output of some other unit in the
computational graph (hidden unit).

Let us consider two cases:
\begin{enumerate}
\item If $f(x_i,0,\theta_i)>0$, the basic state is {\em active},
the unit is generally sensitive and non-zero,
but sometimes it is shut off (e.g., when $z_i$ is sufficiently negative to push 
the argument of the rectifier below 0). In that case gradients will flow
in most cases (samples of $z_i$). If the rest of the system sends the signal 
that $h_i$ should have been smaller, then gradients will push it towards
being more often in the insensitive state.
\item If $f(x_i,0,\theta_i)=0$, the basic state is {\em inactive},
  the unit is generally insensitive and zero,
  but sometimes turned on (e.g., when $z_i$ is sufficiently positive to push the
  argument of the rectifier above 0). In that case gradients will not flow
  in most cases, but when they do, the signal will either push the weighted
  sum lower (if being active was not actually a good thing for that unit in
  that context) and reduce the chances of being active again, or it will
  push the weight sum higher (if being active was actually a good thing for
  that unit in that context) and increase the chances of being active
  again.
\end{enumerate}

So it appears that even though the gradient does not always flow
(as it would with sigmoid or tanh units), it might flow sufficiently
often to provide the appropriate training information. The important
thing to notice is that even when the basic state (second case, above)
is for the unit to be insensitive and zero, {\em there will be
an occasional gradient signal} that can draw it out of there.

One concern with this approach is that one can see an asymmetry between
the number of times that a unit with an active basic state can get a chance
to receive a signal telling it to become inactive, versus the number of
times that a unit with an inactive basic state can get a signal telling
it to become active. 

Another potential and related concern is that some of these units will
``die'' (become useless) if their basic state is inactive in the vast
majority of cases (for example, because their weights pushed them into that
regime due to random fluctations). Because of the above asymmetry, dead
units would stay dead for very long before getting a chance of being born
again, while live units would sometimes get into the death zone by chance
and then get stuck there. What we propose here is a simple mechanism to
{\em adjust the bias of each unit} so that in average its ``firing rate''
(fraction of the time spent in the active state) reaches some pre-defined
target. For example, if the moving average of being non-zero falls below a
threshold, the bias is pushed up until that average comes back above the
threshold.

\section{Unbiased Estimator of Gradient for Stochastic Binary Neurons}
\label{sec:unbiased}

Let us consider the case where we want some component of our model
to take a hard decision but allow this decision to be stochastic,
with a probability that is a continuous function of some 
quantities, through parameters that we wish to learn. 
We will also assume that many such decisions can be taken
in parallel with independent noise sources $z_i$ driving the stochastic
samples. Without loss of generality, we consider here a
set of binary decisions, i.e., the setup corresponds to
having a set of stochastic binary neurons, whose output $h_i$
influences an observed future loss $L$. In the framework of
Eq.~\ref{eq:noisy-output}, we could have for example
\begin{equation}
\label{eq:stochastic-binary-neuron}
  h_i = f(x_i, z_i, \theta_i) = \one_{z_i > \sigm(a_i)}
\end{equation}
where $z_i \sim U[0,1]$ is uniform and $\sigm(u)=1/(1+\exp(-u))$
is the sigmoid function. In the case
of the traditional artificial neuron, we would have
\[
 a_i = b_i + W_i \cdot x_i
\]
and $\theta_i=(b_i,W_i)$
is the set of parameters for unit $i$ (scalar bias $b_i$
and weight vector $W_i$). We would ideally like to estimate
how a change in $a_i$ would impact $L$ in average over the
noise sources, so as to be able to propagate this estimated
gradients into $\theta_i$ and possibly into $x_i$.

\subsection{Derivation of Unbiased Estimator}

\begin{theorem}
\label{thm:unbiased-estimator}
Let $h_i$ be defined as in Eq.~\ref{eq:stochastic-binary-neuron},
then $\hat{g}_i=(h_i - \sigm(a_i)) L$ 
is an unbiased estimator of $g_i=\frac{\partial
  E_{z_i,c_{-i}}[L|c_i]}{\partial a_i}$ where the expectation is over $z_i$
and over all the noise sources $c_{-i}$, besides $z_i$, that do not
influence $a_i$ but may influence $L$, i.e., conditioned on the set of
noise sources $c_i$ that influence $a_i$.
\end{theorem}

\begin{proof}
We will compute the expected value of the estimator and verify
that it equals the desired derivative. The set of all noise sources
in the system is $\{z_i\} \cup c_i \cup c_{-i}$. We can consider
$L$ to be an implicit deterministic function of all the noise sources, i.e.,
$L=L(h_i,c_i,c_{-i})$, where $h_i$ contains everything about $z_i$ that
is needed to predict $L$. $E_{v_z}[\cdot]$ denotes the expectation over variable $v_z$,
while $E[\cdot|v_{z}]$ denotes
the expectation over all the other random variables besides $v_z$, i.e., conditioned on $v_Z$.
\begin{align}
 E[L|c_i] &= E_{c_{-i}}[E_{z_i}[L(h_i,c_i,c_{-i})]] \nonumber \\
                    &= E_{c_{-i}}[E_{z_i}[h_i L(1,c_i,c_{-i})+(1-h_i) L(0,c_i,c_{-i})]] \nonumber \\
                    &=E_{c_{-i}}[P(h_i=1|a_i) L(1,c_i,c_{-i})+P(h_i=0|a_i) L(0,c_i,c_{-i})] \nonumber \\
                    &=E_{c_{-i}}[\sigm(a_i) L(1,c_i,c_{-i})+(1-\sigm(a_i)) L(0,c_i,c_{-i})] 
\end{align}
Since $a_i$ does not influence $P(c_{-i})$, differentiating with respect to $a_i$ gives
\begin{align}
 g_i \defeq \frac{\partial E[L|c_i]}{\partial a_i} &= 
   E_{c_{-i}}[\frac{\partial \sigm(a_i)}{\partial a_i} L(1,c_i,c_{-i})-
            \frac{\partial \sigm(a_i)}{\partial a_i} L(0,c_i,c_{-i}) | c_i]  \nonumber \\
 &=  E_{c_{-i}}[\sigm(a_i)(1-\sigm(a_i))(L(1,c_i,c_{-i})-L(0,c_i,c_{-i}) | c_i] 
\label{eq:gradient}
\end{align}
First consider that since  $h_i \in \{0,1\}$,
\[
 L(h_i,c_i,c_{-i}) = h_i L(1,c_i,c_{-i}) + (1-h_i) L(0,c_i,c_{-i})
\]
$h_i^2=h_i$ and $h_i(1-h_i)=0$, so
\begin{align}
\hat{g}_i \defeq (h_i - \sigm(a_i)) L(h_i,c_i,c_{-i}) &= h_i(h_i - \sigm(a_i)) L(1,c_i,c_{-i}) + (h_i-\sigm(a_i))(1-h_i) L(0,c_i,c_{-i})) \nonumber \\
  &= h_i(1 - \sigm(a_i)) L(1,c_i,c_{-i}) - (1-h_i) \sigm(a_i) L(0,c_i,c_{-i}).
\end{align}
Now let us consider the expected value of the estimator $\hat{g}_i=(h_i - \sigm(a_i)) L(h_i,c_i,c_{-i})$.
\begin{align}
 E[\hat{g}_i] &= 
   E[h_i(1 - \sigm(a_i)) L(1,c_i,c_{-i}) - (1-h_i) \sigm(a_i) L(0,c_i,c_{-i})] \nonumber \\
  & =    E_{c_i,c_{-i}}[\sigm(a_i)(1 - \sigm(a_i)) L(1,c_i,c_{-i}) - (1-\sigm(a_i)) \sigm(a_i) L(0,c_i,c_{-i})] \nonumber \\
  & =    E_{c_i,c_{-i}}[\sigm(a_i)(1 - \sigm(a_i)) (L(1,c_i,c_{-i}) -L(0,c_i,c_{-i}))] 
\end{align}
which is the same as Eq.~\ref{eq:gradient}, i.e., the expected value of the
estimator equals the gradient of the expected loss, $E[\hat{g}_i]=g_i$.
\end{proof}

\begin{corollary}
\label{eq:corollary}
Under the same conditions as Theorem~\ref{thm:unbiased-estimator},
and for any (possibly unit-specific) constant $\bar{L}_i$ 
the centered estimator
\[
   (h_i - \sigm(a_i))(L - \bar{L}_i),
\]
is also an unbiased estimator of 
$g_i=\frac{\partial E_{z_i,c_{-i}}[L|c_i]}{\partial a_i}$. 
Furthermore, among all possible values of $\bar{L}_i$, the minimum 
variance choice is 
\begin{equation}
\label{eq:opt-L}
 \bar{L}_i = \frac{E[(h_i-\sigm(a_i))^2 L]}{E[(h_i-\sigm(a_i))^2]},
\end{equation}
which we note is a weighted average of the loss values $L$, whose
weights are specific to unit $i$.
\end{corollary}
\begin{proof}
The centered estimator $(h_i - \sigm(a_i))(L - \bar{L}_i)$ can be decomposed
into the sum of the uncentered estimator $\hat{g}_i$ and the term
$(h_i - \sigm(a_i))\bar{L}_i$.  Since $E_{z_i}[h_i|a_i]=\sigm(a_i)$, 
$E[\bar{L}_i(h_i - \sigm(a_i))|a_i]=0$, so that
the expected value of the centered estimator equals the
expected value of the uncentered estimator. By
Theorem~\ref{thm:unbiased-estimator} (the uncentered estimator is
unbiased), the centered estimator is therefore also unbiased,
which completes the proof of the first statement.

Regarding the optimal choice of $\bar{L}_i$,
first note that the variance of the uncentered estimator is
\[
 Var[(h_i - \sigm(a_i))L]=E[(h_i - \sigm(a_i))^2L^2] - E[\hat{g}_i]^2.
\]
Now let us compute the variance of the centered estimator:
\begin{align}
  Var[(h_i - \sigm(a_i))(L - \bar{L}_i)] &=& 
  E[(h_i - \sigm(a_i))^2(L - \bar{L}_i)^2] - E[(h_i-\sigma(a_i))(L-\bar{L}_i)]^2 \nonumber \\
 &=& E[(h_i - \sigm(a_i))^2L^2] + E[(h_i-\sigm(a_i))^2\bar{L}_i^2] \nonumber \\
   && - 2 E[(h_i-\sigm(a_i))^2L\bar{L}_i] -(E[\hat{g}_i]-0)^2 \nonumber \\
 &=& Var[(h_i - \sigm(a_i))L] - \Delta 
\end{align}
where $\Delta=2 E[(h_i-\sigm(a_i))^2 L\bar{L}_i] - E[(h_i-\sigm(a_i))^2\bar{L}_i^2]$.
Let us rewrite $\Delta$:
\begin{align}
 \Delta &=& 2 E[(h_i-\sigm(a_i))^2 L \bar{L}_i] - E[(h_i-\sigm(a_i))^2\bar{L}_i^2] \nonumber \\
      &=& E[(h_i-\sigm(a_i))^2 \bar{L}_i(2L - \bar{L}_i)] \nonumber \\
   &=& E[(h_i-\sigm(a_i))^2 (L^2 - (L-\bar{L}_i)^2)] 
\end{align}
$\Delta$ is maximized (to minimize variance of the estimator)
when $E[(h_i-\sigm(a_i))^2 (L-\bar{L}_i)^2]$ is minimized. Taking the
derivative of that expression with respect to $\bar{L}_i$, we obtain
\[
  2 E[(h_i-\sigm(a_i))^2 (\bar{L}_i-L)] =0
\]
which is achieved for 
\[
  \bar{L}_i = \frac{E[(h_i-\sigm(a_i))^2 L]}{E[(h_i-\sigm(a_i))^2]}
\]
as claimed.
\end{proof}
Practically, we could get the lowest variance estimator (among all choices of the $\bar{L}_i$)
by keeping track of two numbers (running or moving averages) 
for each stochastic neuron, one for the numerator
and one for the denominator of the unit-specific $\bar{L}_i$ in Eq.~\ref{eq:opt-L}.
This would lead the lowest-variance estimator
\[
  (h_i - \sigm(a_i))(L - \bar{L}_i).
\]

Note how the unbiased estimator only requires broadcasting $L$ throughout the
network, no back-propagation and only local computation.
Note also how this could be applied even with an estimate
of future rewards or losses $L$, as would be useful in the context of
reinforcement learning (where the actual loss or reward will be measured farther
into the future, much after $h_i$ has been sampled).

\subsection{Training a Lower-Variance Biased Estimator}
\label{sec:biased}

One potential problem with the above unbiased estimators is that their
variance could be large enough to considerably slow training, when compared
to using stochastic gradient descent with back-propagated gradients.

We propose here a general class of solutions to address that challenge,
but for this purpose we need to have a biased but low-variance estimator.

\subsubsection{Biased Low-Variance Estimator}

A plausible unbiased estimator is the developed below. 

Let us $\hat{G}_j$ be an estimator of the gradient of the expected loss
with respect to the activation (pre-nonlinearity) $a_j$ of unit $j$, and
let unit $j$ compute its activation as a deterministic smooth function of
the output $h_i$ of unit $i$ (for example, $a_j = \sum_i W_{ji} h_i$).
Then we can clearly get an estimator of the gradient with respect to $h_i$
by $\sum_j \hat{G}_j \frac{\partial a_j}{\partial h_i}$.  The problem is to
back-propagate through the binary threshold function which produced $h_i$
from the noise $z_i$ and the activation $a_i$
(Eq.~\ref{eq:stochastic-binary-neuron}). The biased estimator
we propose here\footnote{already explored by Goeff Hinton \citep{Hinton-Coursera2012},
lecture 15b }
is simply 
\[
  \hat{G}_i = \sum_j \hat{G}_j \frac{\partial a_j}{\partial h_i}
\]
as the estimator of the gradient of the expected loss with respect to $a_i$,
i.e., we ignore the derivative of the threshold function $f$.

\subsubsection{Combining a Low-Variance High-Bias Estimator
with a High-Variance Low-Bias Estimator}

Let us assume someone hands us two estimators $\hat{G}_i$
and $\hat{g}_i$, with the first having low variance but high bias,
while the second has high variance and low bias. How could we
take advantage of them to obtain a better estimator?

What we propose is the following: {\bf train a function 
${\cal G}_i$ which takes as input the high-bias low-variance estimator $\hat{G}_i$
and {\em predicts} the low-bias high-variance estimator $\hat{g}_i$.}

By construction, since ${\cal G}_i$ is a deterministic function of a
low-variance quantity (we could add other inputs to help it in its
prediction, but they should not be too noisy), it should also have low
variance. Also by construction, and to the extent that the learning task is
feasible, the prediction ${\cal G}_i$ will strive to be as close as
possible to the expected value of the unbiased estimator, i.e., ${\cal G}_i
\rightarrow E[\hat{g}_i | \hat{G}_i]$. It is therefore a way to
unbias $\hat{G}_i$ to the extent that it is possible. Note that adding
appropriate auxiliary inputs to ${\cal G}_i$ could be helpful in this respect.

\section{Efficiency of Reward Correlation Estimators}

One of the questions that future work should address
is the efficiency of estimators such as those introduced above.

\subsection{The Unbiased Estimator as Reward Correlator}
\label{sec:correlation}

In this respect, it is interesting to note how the proposed unbiased estimator (in particular
the centered one) is very similar in form to the just estimating the correlation between
the stochastic decision $h_i$ and the ensuing loss $L$:
\[
  {\rm Correlation} = E[(h_i - E[h_i|c_i])(L - E[L|c_i]) | c_i]
\]
Note how this is the correlation between $h_i$ and $L$ {\em in the
context of the other noise sources $c_i$ that influence $h_i$}.
Note that a particular ``noise source'' is just the input of the model.

\subsection{The Boltzmann Machine Gradient as Unnormalized Reward Correlation}
\label{sec:BM}

The log-likelihood gradient over a bias (offset) parameter $b_i$ associated with
a unit $X_i$ (visible or hidden) 
of a Boltzmann machine with distribution $P$ and a training example $v$ 
(associated with visible units $V$) is 
\[
   \frac{\partial \log P(V=v)}{\partial b_i} = E[X_i | V=v] - E[X_i]
\]
where the expectation is over the model's distribution, which
defines a joint distribution between all the units of the model,
including the visible ones. The conditional distribution of
binary unit $X_i$ given the other units is given by
\[
  P(X_i=1 | X_{-i}) = \sigm(a_i) = \sigm(b_i + \sum_{j\neq i} W_{ij} X_j).
\]
where $a_i$ is the unit activation (prior to applying the sigmoid).
An unbiased estimator of the above gradient is
\[
   X_i^+ - X_i^-
\]
where $X^+$ is a configuration obtained while $V$ was clamped to
a training example $v$, and $X^-$ is a configuration obtained without
this constraint.
Similarly, the log-likelihood gradient over weight $W_{ij}$
is estimated unbiasedly by
\[
  X_i^+ X_j^- - X_i^- X_j^-.
\]
We call these estimators the {\em Boltzmann machine log-likelihood gradient estimators}.

We now show that this log-likelihood gradient can also be interpreted 
as an {\em unnormalized} reward correlator in a particular setup, where the objective is to
discriminate between examples coming from the data generating distribution
and examples from the model distribution $P$. 

The setup is the following.
Let the model generate samples $X=(V,H)$, where $V$ are the visible
units and $H$ the hidden units. Let the reward for generating a ``negative example''
$V=V^-$ in this way be $R=-1$. However, let us toss a coin to select from this stream some 
samples $(V,H)$ such that we can declare $X=X^+$ as coming from the training 
distribution of interest $\pi$, for example using rejection sampling.
Then we let the reward be $R=1$ because the model has generated a ``good example'' $V^+$. 
Although this might be a rare event let us randomly choose among the negative examples $V^-$
so that the average number of negative examples equals the average number of
positive examples $V^+$. Clearly the samples $V^+$ follow the training distribution $\pi$ while
the samples $V^-$ follow the model distribution $P$.
In our setup, let us imagine that $X$'s were obtained by Gibbs sampling.
The Gibbs chain corresponds to constructing a {\em computational graph}
that deterministically or stochastically computes various quantities
(one per node of the graph), here the activation $a_{it}=b_i + \sum_{j\neq i} W_{ij} X_{j{t-1}}$ 
and the binomially sampled bit $X_{it} \sim Bin(\sigm(a_{it}))$ 
for each stochastic unit at each step $t$ of the chain.

Using Theorem~\ref{thm:unbiased-estimator},
 $(X_{it} - \sigm(a_{it})) R$ is an unbiased estimator of 
\[
\frac{\partial E[R | c_{it}]}{\partial a_{it}}.
\]

Now we consider an {\em unnormalized} variant of that estimator,
\[
  X_{it} R
\]  
where $X$ (rather than its centered version)
is multiplied by $R$.
This means in particular that $X_{it} R$
is a (possibly biased) estimator of the gradient with respect to $b_i$, while, by the chain rule
$X_{it} R X_{j,{t-1}}$
is an estimator of the gradient with respect to $W_{ij}$.

If we separate the training examples into those with $R=1$ and those with $R=-1$,
we obtain the estimator associated with a $(V^+,V^-)$ pair,
\[
  X_i^+ - X_i^-
\]
for biases $b_i$, and
\[
  X_i^+ X_j^+ - X_i^-X_j^-
\]
for weights $W_{ij}$. The above two estimators of the gradients
$\frac{\partial E[R | c_{it}]}{\partial b_i}$ and
$\frac{\partial E[R | c_{it}]}{\partial W_{ij}}$ correspond
exactly to the Boltzmann machine log-likelihood
gradient estimators.

\section{Conclusion}

In this paper, we have motivated estimators of the gradient through highly
non-linear non-differentiable functions (such as those corresponding to an
indicator function), especially in networks involving noise sources, such
as neural networks with stochastic neurons. They can be useful as
biologically motivated models and they might be useful for engineering
(computational efficiency) reasons when trying to reduce computation via
conditional computation or to reduce interactions between parameters via
sparse updates~\citep{Bengio-arxiv-2013}. We have discussed a general class
of stochastic neurons for which the gradient is exact for given fixed noise
sources, but where the non-linearity is not saturating on all of its range
(semi-hard stochastic neurons), and for which ordinary back-prop can be
used. We have also discussed the case of completely saturating
non-linearities, for which we have demonstrated the existence of an
unbiased estimator based on correlating the perturbation with the observed
reward, which is related to but {\em different} from the
SPSA~\citep{Spall-1992} estimator. Indeed the latter {\em divides} the
change in reward by the perturbation, instead of multiplying them.

We have shown that it was possible in principle to obtain lower variance
estimators, in particular by training a function to turn a biased estimator
but low variance estimator into one that is trained to be unbiased but has
lower variance. We have also shown that the Boltzmann machine gradient
could be interpreted as a particular form of reward correlation (without
the additive normalization). Since training of Restricted Boltzmann
Machines has been rather successful, this suggests that correlation-based
estimators might actually work well in practice. Clearly, future work
should investigate the relative practical merits of these estimators,
and in particular how their variance scale with respect to the number
of independent noise sources (e.g., stochastic neurons) present in
the system.

\subsection*{Acknowledgements}

The author would like to thank NSERC, the Canada Research Chairs, and
CIFAR for support.

\bibliography{strings-shorter,ml,aigaion-shorter}
\bibliographystyle{natbib}

\end{document}